\documentclass{article} 
\usepackage{iclr2021_conference,times}

\usepackage{amsmath,amsfonts,bm}

\def\eqref#1{equation~\ref{#1}}

\def\1{\bm{1}}

\def\va{{\bm{a}}}

\DeclareMathAlphabet{\mathsfit}{\encodingdefault}{\sfdefault}{m}{sl}
\SetMathAlphabet{\mathsfit}{bold}{\encodingdefault}{\sfdefault}{bx}{n}

\newcommand{\Var}{\mathrm{Var}}

\usepackage{hyperref}
\usepackage{url}

\usepackage[utf8]{inputenc} 
\usepackage[T1]{fontenc}    
\usepackage{hyperref}       
\usepackage{url}            
\usepackage{booktabs}       
\usepackage{amsfonts}       
\usepackage{nicefrac}       
\usepackage{microtype}      

\usepackage{microtype}
\usepackage{graphicx}
\usepackage{subfigure}
\usepackage{booktabs} 

\usepackage{color}
\usepackage{mathtools}
\usepackage{algorithm} 

\usepackage{algpseudocode}  
\usepackage{subfigure}
\usepackage{multirow}
\usepackage{makecell}
\usepackage{array, booktabs, arydshln, xcolor}
\usepackage{amssymb}
\usepackage{graphbox}
\usepackage{url}
\usepackage{algorithm} 
\usepackage{amsthm}
\usepackage{algpseudocode}

\usepackage{amsthm}

\usepackage{thmtools}
\usepackage{thm-restate}
\usepackage{hyperref}
\usepackage{cleveref}

\declaretheorem[name=Lemma]{lemma}
\allowdisplaybreaks

\usepackage{color}
\usepackage{mathtools}

\usepackage{wrapfig, blindtext}
\usepackage{subfigure}
\usepackage{multirow}
\usepackage{makecell}
\usepackage{array, booktabs, arydshln, xcolor}
\usepackage{amssymb}
\usepackage{graphbox}
\usepackage{caption}
\allowdisplaybreaks

\newcommand{\shortn}{\textup{\texttt{-}}}
\newcommand{\shorte}{\textup{\texttt{=}}}
\newcommand{\shortp}{\textup{\texttt{+}}}
\newcommand{\ie}{\textit{i}.\textit{e}.}

\newcommand{\name}{DOP}
\newcommand{\Tau}{\mathrm{T}}

\title{DOP: Off-Policy Multi-Agent Decomposed \\ Policy Gradients}
\makeatletter
\newcommand{\printfnsymbol}[1]{%
  \textsuperscript{\@fnsymbol{#1}}%
}
\author{%
  Yihan Wang\thanks{Equal Contribution. Listing order is random.} \ , Beining Han\printfnsymbol{1}, Tonghan Wang\printfnsymbol{1}, Heng Dong, Chongjie Zhang \\
  Institute for Interdisciplinary Information Sciences\\
  Tsinghua University, Beijing, China\\
  \texttt{\{memoryslices,bouldinghan,tonghanwang1996,drdhii\}@gmail.com} \\
  \texttt{chongjie@tsinghua.edu.cn}
}

\iclrfinalcopy 
\begin{document}

\maketitle
\renewcommand\UrlFont{\rmfamily\itshape}
\begin{abstract}
Multi-agent policy gradient (MAPG) methods recently witness vigorous progress. However, there is a significant performance discrepancy between MAPG methods and state-of-the-art multi-agent value-based approaches. In this paper, we investigate causes that hinder the performance of MAPG algorithms and present a multi-agent decomposed policy gradient method (\name). This method introduces the idea of value function decomposition into the multi-agent actor-critic framework. Based on this idea, \name~supports efficient off-policy learning and addresses the issue of \emph{centralized-decentralized mismatch} and credit assignment in both discrete and continuous action spaces. We formally show that \name~critics have sufficient representational capability to guarantee convergence. In addition, empirical evaluations on the StarCraft II micromanagement benchmark and multi-agent particle environments demonstrate that \name~significantly outperforms both state-of-the-art value-based and policy-based multi-agent reinforcement learning algorithms. Demonstrative videos are available at \url{https://sites.google.com/view/dop-mapg/}.
\end{abstract}

\section{Introduction}
Cooperative multi-agent reinforcement learning (MARL) has achieved great progress in recent years~\citep{hughes2018inequity, jaques2019social, vinyals2019grandmaster, zhang2019multi, baker2020emergent, wang2020roma}. Advances in valued-based MARL~\citep{sunehag2018value, rashid2018qmix, son2019qtran, wang2020learning} contribute significantly to the progress, achieving state-of-the-art performance on challenging tasks, such as StarCraft II micromanagement~\citep{samvelyan2019starcraft}. However, these value-based methods present a major challenge for stability and convergence in multi-agent settings~\citep{wang2020understanding}, which is further exacerbated in continuous action spaces. Policy gradient methods hold great promise to resolve these challenges. MADDPG~\citep{lowe2017multi} and COMA~\citep{foerster2018counterfactual} are two representative methods that adopt the paradigm of centralized critic with decentralized actors (CCDA), which not only deals with the issue of non-stationarity~\citep{foerster2017stabilising, hernandezleal2017survey} by conditioning the centralized critic on global history and actions but also maintains scalable decentralized execution via conditioning policies on local history. Several subsequent works make improvements to the CCDA framework by introducing the mechanism of recursive reasoning~\citep{wen2019probabilistic} or attention~\citep{iqbal2019actor}. 

Despite the progress, most of the multi-agent policy gradient (MAPG) methods do not provide satisfying performance, e.g., significantly underperforming value-based methods on benchmark tasks~\citep{samvelyan2019starcraft}. In this paper, we analyze this discrepancy and pinpoint three major issues that hinder the performance of MAPG methods. (1) Current stochastic MAPG methods do not support off-policy learning, partly because using common off-policy learning techniques is computationally expensive in multi-agent settings. (2) In the CCDA paradigm, the suboptimality of one agent's policy can propagate through the centralized joint critic and negatively affect policy learning of other agents, causing catastrophic miscoordination, which we call \emph{centralized-decentralized mismatch} (CDM). (3) For deterministic MAPG methods, realizing efficient credit assignment~\citep{tumer2002learning, agogino2004unifying} with a single global reward signal largely remains challenging. 

In this paper, we find that these problems can be addressed by introducing the idea of value decomposition into the multi-agent actor-critic framework and learning a centralized but factorized critic. This framework decomposes the centralized critic as a weighted linear summation of individual critics that condition on local actions. This decomposition structure not only enables scalable learning on the critic, but also brings several benefits. It enables tractable off-policy evaluations of stochastic policies, attenuates the CDM issues, and also implicitly learns an efficient multi-agent credit assignment. Based on this decomposition, we develop efficient off-policy multi-agent stochastic policy gradient methods for both discrete and continuous action spaces.

A drawback of an linearly decomposed critic is its limited representational capacity~\citep{wang2020qplex}, which may induce bias in value estimations. However, we show that this bias does not violate the policy improvement guarantee of policy gradient methods and that using decomposed critics can largely reduce the variance in policy updates. In this way, a decomposed critic achieves a great bias-variance trade-off. 


We evaluate our methods on both the StarCraft II micromanagement benchmark~\citep{samvelyan2019starcraft} (discrete action spaces) and multi-agent particle environments~\citep{lowe2017multi, mordatch2018emergence} (continuous action spaces). Empirical results show that \name~is very stable across different runs and outperforms other MAPG algorithms by a wide margin. Moreover, to our best knowledge, stochastic \name~provides the first MAPG method that significantly outperforms state-of-the-art valued-based methods in discrete-action benchmark tasks.

\textbf{Related works on value decomposition methods.} In value-based MARL, value decomposition~\citep{guestrin2002coordinated, castellini2019representational} is widely used. These methods learn local Q-value functions for each agent, which are combined with a learnable mixing function to produce global action values. In VDN~\citep{sunehag2018value}, the mixing function is an arithmetic summation. QMIX~\citep{rashid2018qmix, rashid2020monotonic} proposes a non-linear monotonic factorization structure. QTRAN~\citep{son2019qtran} and QPLEX~\citep{wang2020qplex} further extend the class of value functions that can be represented. NDQ~\citep{wang2020learning} addresses the miscoordination problem by learning nearly decomposable architectures. In this paper, we study how value decomposition can be used to enable efficient multi-agent policy-based learning. In Appendix F, we discuss how \name~is related to recent progress in multi-agent reinforcement learning and provide detailed comparisons with existing multi-agent policy gradient methods.


\section{Background}\label{sec:background}
We consider fully cooperative multi-agent tasks that can be modelled as a Dec-POMDP~\citep{oliehoek2016concise} $G\shorte\langle I, S, A, P, R, \Omega, O, n, \gamma\rangle$, where $I$ is the finite set of agents, $\gamma\in[0, 1)$ is the discount factor, and $s\in S$ is the true state of the environment. At each timestep, each agent $i$ receives an observation $o_i\in \Omega$ drawn according to the observation function $O(s, i)$ and selects an action $a_i\in A$, forming a joint action $\va\in A^n$, leading to a next state $s'$ according to the transition function $P(s'|s, \va)$ and a reward $r=R(s,\va)$ shared by all agents. Each agent learns a policy $\pi_i(a_i | \tau_i; \theta_i)$, which is parameterized by $\theta_i$ and conditioned on the local history $\tau_i\in \Tau\equiv(\Omega\times A)^*$. The joint policy $\bm{\pi}$, with parameters $\theta = \langle\theta_1, \cdots, \theta_n\rangle$, induces a joint action-value function: $Q_{tot}^{\bm{\pi}}(\bm\tau$,$ \va)$=$\mathbb{E}_{s_{0:\infty},\va_{0:\infty}}[\sum_{t=0}^\infty \gamma^{t}R(s_t, \va_t)|$ $s_0\shorte s,\va_0\shorte \va,\bm{\pi}]$. We consider both discrete and continuous action spaces, for which stochastic and deterministic policies are learned, respectively. To distinguish deterministic policies, we denote them by $\bm\mu=\langle\mu_1, \cdots, \mu_n\rangle$.

\textbf{Multi-Agent Policy Gradients} The \emph{centralized training with decentralized execution} (CTDE) paradigm~\citep{foerster2016learning, wang2020influence} has recently attracted attention for its ability to address non-stationarity while maintaining decentralized execution. Learning a centralized critic with decentralized actors (CCDA) is an efficient approach that exploits the CTDE paradigm. MADDPG and COMA are two representative examples. MADDPG~\citep{lowe2017multi} learns deterministic policies in continuous action spaces and uses the following gradients to update policies:
\begin{equation}\label{equ:maddpg}
    g = \mathbb{E}_{\bm{\tau},\va\sim\mathcal{D}}\left[\sum_i \nabla_{\theta_i} \mu_i(\tau_i)\nabla_{a_i}Q_{tot}^{\bm{\mu}}(\bm\tau,\va)|_{a_i=\mu_i(\tau_i)}\right],
\end{equation}
and $\mathcal{D}$ is a replay buffer. COMA~\citep{foerster2018counterfactual} updates stochastic policies using the gradients:
\begin{equation}\label{equ:coma}
    g = \mathbb{E}_{\bm{\pi}}\left[\sum_i\nabla_{\theta_i}\log \pi_i(a_i|\tau_i)A_i^{\bm\pi}(\bm\tau,\va)\right],
\end{equation}
where $A^{\bm\pi}_i(\bm \tau,\va) = Q^{\bm{\pi}}_{tot}(\bm \tau,\va)-\sum_{a_i'}Q^{\bm{\pi}}_{tot}(\bm \tau,(\va_{\shortn i}, a_i'))$ is a counterfactual advantage ($\va_{\shortn i}$ is the joint action other than agent $i$) that deals with the issue of credit assignment and reduces variance.

\section{Analysis}\label{sec:analysis}
In this section, we investigate challenges that limit the performance of state-of-the-art multi-agent policy gradient methods.
\subsection{Off-Policy Learning for Multi-Agent Stochastic Policy Gradients}\label{sec:ana-off}
Efficient stochastic policy learning in single-agent settings relies heavily on using off-policy data~\citep{lillicrap2015continuous, wang2016sample, fujimoto2018addressing, haarnoja2018soft}, which is not supported by existing stochastic MAPG methods~\citep{foerster2018counterfactual}. In the CCDA framework, off-policy policy evaluation---estimating $Q_{tot}^{\bm\pi}$ from data drawn from behavior policies $\bm\beta=\langle\beta_1,\dots,\beta_n\rangle$---encounters major challenges. Importance sampling~\citep{meuleau2000off, jie2010connection, levine2013guided} is a simple way to correct for the discrepancy between $\bm\pi$ and $\bm\beta$, but, it requires computing $\prod_i \frac{\pi_i(a_i|\tau_i)}{\beta_i(a_i|\tau_i)}$, whose variance grows exponentially with the number of agents in multi-agent settings. An alternative is to extend the tree backup technique~\citep{precup2000eligibility, munos2016safe} to multi-agent settings and use the $k$-step tree backup update target for training the critic:
\begin{equation}\label{equ:ma_tb}
    y^{TB} = Q_{tot}^{\bm\pi}(\bm\tau,\va) + \sum_{t=0}^{k-1} \gamma^{t} \left(\prod_{l=1}^{t}\lambda\bm\pi(\va_l|\bm\tau_l)\right)\left[r_{t}+\gamma \mathbb{E}_{\bm\pi} [Q^{\bm\pi}_{tot}(\bm\tau_{t+1},\cdot)]-Q_{tot}^{\bm\pi}(\bm\tau_{t}, \va_{t})\right],
\end{equation}
where $\bm\tau_0=\bm\tau$, $\va_0=\va$. However, the complexity of computing $\mathbb{E}_{\bm\pi} [Q^{\bm\pi}_{tot}(\bm\tau_{t+1},\cdot)]$ is $O(|A|^n)$, which becomes intractable when the number of agents is large. Therefore, it is challenging to develop off-policy stochastic MAPG methods. 

\subsection{The Centralized-Decentralized Mismatch Issue}\label{sec:ana-CDM}
In the centralized critic with decentralized actors (CCDA) framework, agents learn individual policies, $\pi_i(a_i | \tau_i ; \theta_i)$, conditioned on the local observation-action history. However, the gradients for updating these policies are dependent on the centralized joint critic, $Q^{\bm\pi}_{tot}(\bm\tau, \va)$ (see Eq.~\ref{equ:maddpg} and~\ref{equ:coma}), which introduces the influence of actions of other agents. Intuitively, gradient updates will move an agent in the direction that can increase the global Q value, but the presence of other agents' actions incurs large variance in the estimates of such directions. 

Formally, the variance of policy gradients for agent $i$ at $(\tau_i, a_i)$ is dependent on other agents' actions:
\begin{equation}
\begin{aligned}
  & \text{Var}_{\va_{\shortn i}\sim\bm\pi_{\shortn i}}\left[Q^{\bm\pi}_{tot}(\bm\tau, (a_i, \va_{\shortn i}))\nabla_{\theta_i}\log \pi_i(a_i|\tau_i)\right] \\
= & \text{Var}_{\va_{\shortn i}\sim\bm\pi_{\shortn i}}\left[Q^{\bm\pi}_{tot}(\bm\tau, (a_i, \va_{\shortn i}))\right](\nabla_{\theta_i}\log \pi_i(a_i|\tau_i))(\nabla_{\theta_i}\log \pi_i(a_i|\tau_i))^\Tau,
\end{aligned}
\end{equation}
where $\text{Var}_{\va_{\shortn i}}\left[Q^{\bm\pi}_{tot}(\bm\tau, (a_i, \va_{\shortn i}))\right]$ can be very large due to the exploration or suboptimality of other agents' policies, which may cause suboptimality in individual policies. For example, suppose that the optimal joint action under $\bm\tau$ is $\va^*$= $\langle a_1^*, \dots, a_n^*\rangle$. When $\mathbb{E}_{\va_{\shortn i}\sim\bm\pi_{\shortn i}}[Q^{\bm\pi}_{tot}(\bm\tau, (a_i^*, \va_{\shortn i}))]<0$ , $\pi_i(a_i^*|\tau_i)$ will decrease, possibly resulting in a suboptimal $\pi_i$. This becomes problematic because a negative feedback loop is created, in which the joint critic is affected by the suboptimality of agent $i$, which disturbs policy updates of other agents. We call this issue \emph{centralized-decentralized mismatch} (CDM).

\textbf{Does CDM occur in practice for state-of-the-art algorithms?} To answer this question, we carry out a case study in Sec.~\ref{sec:exp-cdm}. We can see that the variance of \name~gradients is significantly smaller than COMA and MADDPG (Fig.~\ref{fig:lc_tradeoff} left). This smaller variance enables \name~to outperform other algorithms (Fig.~\ref{fig:lc_tradeoff} middle). We will explain this didactic example in detail in Sec.~\ref{sec:exp-cdm}. In Sec.~\ref{sec:exp-sc2} and~\ref{sec:exp-mpe}, we further show that CDM is exacerbated in sequential decision-making settings, causing divergence even after a near-optimal strategy has been learned.


\subsection{Credit Assignment for Multi-Agent Deterministic Policy Gradients}
MADDPG~\citep{lowe2017multi} and MAAC~\citep{iqbal2019actor} extend deterministic policy gradient algorithms~\citep{silver2014deterministic, lillicrap2015continuous} to multi-agent settings, enabling efficient off-policy learning in continuous action spaces. However, they leave the issue of credit assignment~\citep{tumer2002learning, agogino2004unifying} largely untouched in fully cooperative settings, where agents learn policies from a single global reward signal. In stochastic cases, COMA assigns credits by designing a counterfactual baseline (Eq.~\ref{equ:coma}). However, it is not straightforward to extend COMA to deterministic policies, since the output of polices is no longer a probability distribution. As a result, it remains challenging to realize efficient credit assignment in deterministic cases. 
\section{Decomposed Off-Policy Policy Gradients}
To address the limitations of existing MAPG methods discussed in Sec.~\ref{sec:analysis}, we introduce the idea of value decomposition into the multi-agent actor-critic framework and propose a \emph{Decomposed Off-Policy policy gradient} (\name) method. We factor the centralized critic as a weighted summation of individual critics across agents:
\begin{equation}\label{equ:dc}
    Q^{\phi}_{tot}(\bm \tau, \mathbf{a}) = {\textstyle\sum}_i k_i(\bm \tau) Q_i^{\phi_i}(\bm\tau, a_i) + b(\bm \tau),
\end{equation}
where $\phi$ and $\phi_i$ are parameters of the global and local Q functions, respectively, and $k_i \ge 0$ and $b$ are generated by learnable networks whose inputs are global observation-action histories. In the following sections, we show that this linear decomposition helps address existing limitations of previous methods. A concern is the limited expressivity of linear decomposition~\citep{wang2020qplex}, which may introduce bias in value estimations. We will show that this limitation does not violate the policy improvement guarantee of \name.


\begin{wrapfigure}{r}{0.36\linewidth}
    \centering
    \includegraphics[width=\linewidth]{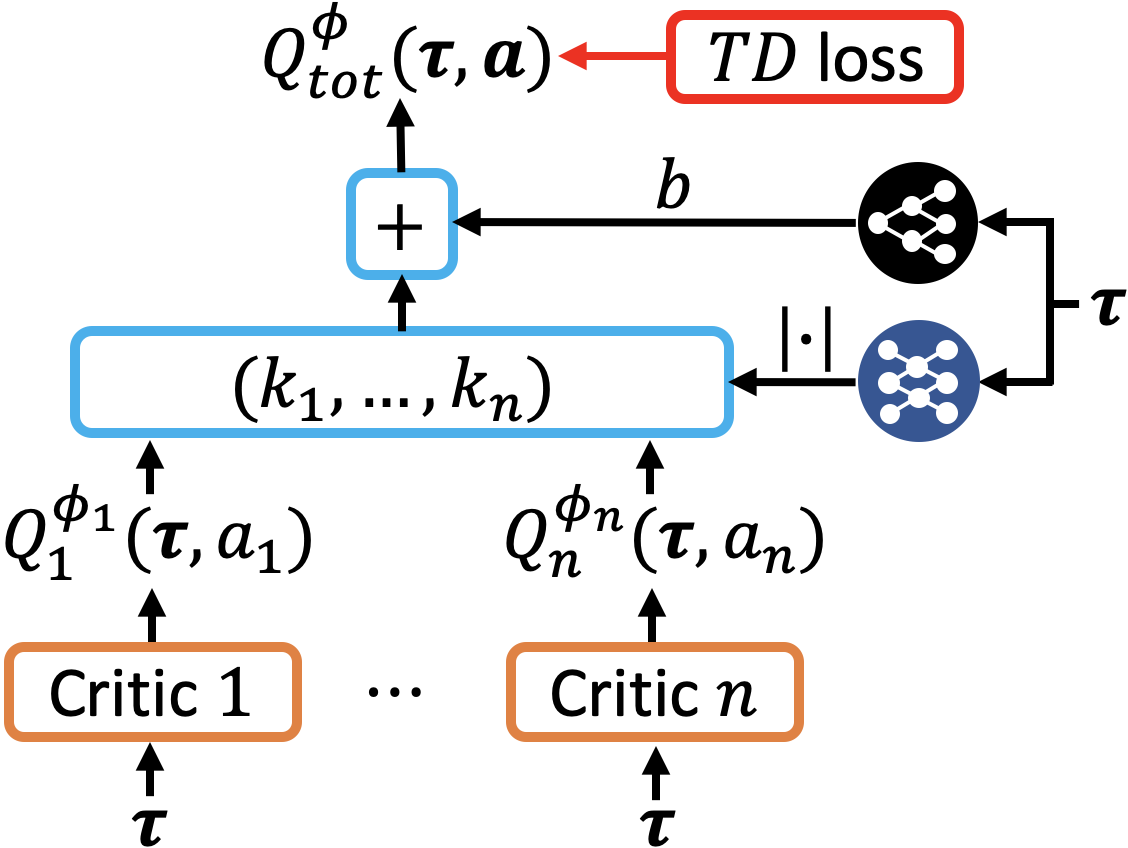}
    \caption{A \textsc{Decomposed} critic.}
    \label{fig:critic}
\end{wrapfigure}
Fig.~\ref{fig:critic} shows the architecture for learning decomposed critics. We learn individual critics $Q_i^{\phi_i}$ by backpropagating gradients from global TD updates dependent on the joint global reward, i.e., $Q_i^{\phi_i}$ is learned implicitly rather than from any reward specific to agent $i$. We enforce $k_i \ge 0$ by applying an absolute activation function at the last layer of the network. The network structure is described in detail in Appendix~\ref{appx:architecture}.


Based on the critic decomposition learning, the following sections will introduce decomposed off-policy policy gradients for learning stochastic policies and deterministic policies, respectively. Similar to other actor-critic methods, \name~alternates between \emph{policy evaluation}---estimating the value function for a policy---and \emph{policy improvement}---using the value function to update the policy~\citep{barto1983neuronlike}.  

\subsection{Stochastic Decomposed Off-Policy Policy Gradients}\label{sec:s_dop}
For learning stochastic policies, the linearly decomposed critic plays an essential role in enabling tractable multi-agent tree backup for off-policy policy evaluation and attenuating the CDM issue while maintaining provable policy improvement.
\subsubsection{Off-Policy Learning}\label{sec:s_dop_tb}
\textbf{Policy Evaluation: Train the Critic} As discussed in Sec.~\ref{sec:ana-off}, using tree backup (Eq.~\ref{equ:ma_tb}) to carry out multi-agent off-policy policy evaluation requires calculating $\mathbb{E}_{\bm\pi} [Q^{\phi}_{tot}(\bm\tau_{t+1},\cdot)]$, which needs $O(|A|^n)$ steps of summation when a joint critic is used. Fortunately, using the linearly decomposed critic, \name~reduces the complexity of computing this expectation to $O(n|A|)$:
\begin{equation}
    \mathbb{E}_{\bm\pi}[Q_{tot}^\phi(\bm\tau, \cdot)] = {\textstyle\sum}_i k_i(\bm\tau)\mathbb{E}_{\pi_i}[Q_i^{\phi_i}(\bm\tau, \cdot)] + b(\bm\tau),
\end{equation}
making the tree backup technique tractable (detailed proof can be found in Appendix A.1). Another challenge of using multi-agent tree backup (Eq.~\ref{equ:ma_tb}) is that the coefficient $c_t=\prod_{l=1}^{t}\lambda\bm\pi(\va_l|\bm\tau_l)$ decays as $t$ gets larger, which may lead to relatively lower training efficiency. To solve this issue, we propose to mix off-policy tree backup updates with on-policy $TD(\lambda)$ updates to trade off sample efficiency and training efficiency. Formally, \name~minimizes the following loss for training the critic:
\begin{equation}\label{equ:s_dop_ct}
    \mathcal{L}(\phi) = \kappa \mathcal{L}_{\bm\beta}^{\text{\name\shortn TB}}(\phi) + (1-\kappa)\mathcal{L}_{\bm\pi}^{\text{On}}(\phi)
\end{equation}
where $\kappa$ is a scaling factor, $\bm\beta$ is the joint behavior policy, and $\phi$ is the parameters of the critic. The first loss item is $\mathcal{L}_{\bm\beta}^{\text{\name\shortn TB}}(\phi) = \mathbb{E}_{\bm\beta} [(y^{\text{\name\shortn TB}} - Q_{tot}^{\phi}(\bm\tau,\va))^2]$, where $y^{\text{\name\shortn TB}}$ is the update target of the proposed $k$-step decomposed multi-agent tree backup algorithm:
\begin{equation}
    y^{\text{\name\shortn TB}} = Q_{tot}^{\phi'}(\bm\tau,\va) + \sum_{t\shorte 0}^{k\shortn 1} \gamma^{t} c_t\left[r_{t}+\gamma \sum_i k_i(\bm\tau_{t\shortp 1})\mathbb{E}_{\pi_i}[Q_i^{\phi_i'}(\bm\tau_{t\shortp 1}, \cdot)] + b(\bm\tau_{t\shortp 1})-Q_{tot}^{\phi'}(\bm\tau_{t}, \va_{t})\right].
\end{equation}
Here, $\phi'$ is the parameters of a target critic, and $\va_t\sim \bm\beta(\cdot|\bm\tau_t)$. The second loss item is $\mathcal{L}_{\bm\pi}^{\text{On}}(\phi) = \mathbb{E}_{\bm\pi} [(y^{\text{On}} - Q_{tot}^{\phi}(\bm\tau,\va))^2]$, where $y^{\text{On}}$ is the on-policy update target as in $\textit{TD}(\lambda)$:
\begin{equation}
    y^{\text{On}} = Q_{tot}^{\phi'}(\bm\tau,\va) + \sum_{t=0}^{\infty} (\gamma\lambda)^{t} \left[r_{t}+\gamma Q_{tot}^{\phi'}(\bm\tau_{t+1}, \va_{t+1})-Q_{tot}^{\phi'}(\bm\tau_{t}, \va_{t})\right].
\end{equation}
In practice, we use two buffers, an on-policy buffer for computing $\mathcal{L}_{\bm\pi}^{\text{On}}(\phi)$ and an off-policy buffer for estimating $\mathcal{L}_{\bm\beta}^{\text{\name\shortn TB}}(\phi)$. 

\textbf{Policy Improvement: Train Actors} Using the linearly decomposed critic architecture, we can derive the following on-policy policy gradients for learning stochastic policies:
\begin{equation}\label{equ:s_dop_g}
    g=\mathbb{E}_{\bm\pi}\left[{\textstyle\sum}_i k_i(\bm\tau) \nabla_{\theta_{i}}\log{\pi_i(a_{i}|\tau_{i}; \theta_i)}Q^{\phi_i}_{i}(\bm\tau,a_{i})\right].
\end{equation}
In Appendix~\ref{appx:stochastic_pg}, we provide the detailed derivation and an off-policy version of stochastic policy gradients. This update rule reveals two important insights. (1) With a linearly decomposed critic, each agent's policy update only depends on the individual critic $Q^{\phi_i}_{i}$. (2) Learning the decomposed critic implicitly realizes multi-agent credit assignment, because the individual critic provides credit information for each agent to improve its policy in the direction of increasing the global expected return. Moreover, Eq.~\ref{equ:s_dop_g} is also the policy gradients when assigning credits via the aristocrat utility~\citep{wolpert2002optimal} (Appendix~\ref{appx:stochastic_pg}). Eq.~\ref{equ:s_dop_ct} and~\ref{equ:s_dop_g} form the core of our \name~algorithm for learning stochastic policies, which we call \emph{stochastic \name} and is described in detail in Appendix~\ref{appx:algorithm}.

\textbf{The CDM Issue} occurs when decentralized policies' suboptimality reinforces each other through the joint critic. As an agent's stochastic \name~gradients do not rely on the actions of other agents, they attenuate the effect of CDM. We empirically show that \name~can reduce variance in policy gradients in Sec.~\ref{sec:exp-cdm} and can attenuate the CDM issue in complex tasks in Sec.~\ref{sec:exp-sc2-ablation}.

\subsubsection{Stochastic \name~Policy Improvement Theorem}\label{sec:s_dop-pit}
In this section, we theoretically demonstrate that stochastic \name~can converge to local optimal despite the fact that a linearly decomposed critic has limited representational capability. Since an accurate analysis for a complex function approximator (e.g., using neural network) is difficult, we adopt several mild assumptions as in previous work~\citep{feinberg2018model}: we analyze $\pi_i$ and Q function with tabular expression and simplify value evaluation as a MSE problem. By Fact~\ref{fact:mono}, we first show that the linearly decomposed structure ensures that the learned local value function $Q_i^{\phi_i}$ preserves the order of $Q^{\bm\pi}_i(\bm\tau, a_i) = \sum_{\bm  a_{\shortn i}} \bm\pi_{\shortn i}(\bm a_{\shortn i} | \bm \tau_{\shortn i}) Q_{tot}^{\bm\pi}(\bm \tau, (a_i, \bm a_{\shortn i}))$.
\begin{restatable}{fact}{mono} \label{fact:mono}
When policy evaluation converges, $Q_i^{\phi_i}$ satisfies:
\begin{equation*}
    Q_i^{\bm\pi}(\bm\tau, a_i) > Q_i^{\bm\pi}(\bm \tau, a_i') \iff Q_i^{\phi_i}(\bm \tau, a_i) > Q_i^{\phi_i}(\bm \tau, a_i'), \ \ \  \forall \bm \tau, i, a_i, a_i'.
\end{equation*}
\end{restatable}

Based on Fact~\ref{fact:mono}, we prove the following theorem to show that even without accurate estimates of $Q_{tot}^{\bm\pi}$, the stochastic \name~policy updates can still monotonically improve the objective $J(\bm\pi) = \mathbb{E}_{\bm\pi} [\sum_t \gamma^t r_t]$.
\begin{restatable}{prop}{sdoppit}\label{prop:s_dop_pit}
[Stochastic \name~policy improvement theorem] Under a mild assumption, for any pre-update policy $\bm\pi^o$ which is updated by Eq.~\ref{equ:s_dop_g} to $\bm\pi$, let $\pi_i(a_i | \tau_i) = \pi_i^o(a_i | \tau_i) + \beta_{a_i, \bm\tau} \delta$, where $\delta>0$ is a sufficiently small number. If it holds that $\forall \tau, a_i', a_i, Q^{\phi_i}_i(\bm \tau, a_i) > Q^{\phi_i}_i(\bm \tau, a_i') \iff \beta_{a_i, \bm\tau} \geq \beta_{a_i', \bm\tau}$, then we have
\begin{equation}
    J(\bm\pi) \geq J(\bm\pi^o),
\end{equation}
\ie, the joint policy is improved by the update.
\end{restatable}

Please refer to Appendix~\ref{appx:s_dop_pit} for the proof of Fact~\ref{fact:mono} and Proposition~\ref{prop:s_dop_pit} and more detailed discussions and remarks on the assumptions we adopted.


\subsection{Deterministic Decomposed Off-Policy Policy Gradients}
\subsubsection{Off-Policy Learning}\label{sec:d_dop_pg}
To enable efficient learning with continuous actions, we propose \emph{deterministic \name}. As in single-agent settings, because deterministic policy gradient methods avoid the integral over actions, it greatly eases the cost of off-policy learning~\citep{silver2014deterministic}. For \textbf{policy evaluation}, we train the critic by minimizing the following TD loss:
\begin{equation}\label{equ:d_dpg_c}
    \mathcal{L}(\phi) = \mathbb{E}_{(\bm\tau_t, r_t, \va_t, \bm\tau_{t+1})\sim\mathcal{D}} \left[\left(r_t + \gamma Q_{tot}^{\phi'}(\bm\tau_{t+1}, \bm\mu(\bm\tau_{t+1}; \theta')) - Q_{tot}^{\phi}(\bm\tau_t, \va_t)\right)^2\right],
\end{equation}
where $\mathcal{D}$ is a replay buffer, and $\phi'$, $\theta'$ are the parameters of the target critic and actors, respectively. For \textbf{policy improvement}, we derive the following deterministic \name~policy gradients:
\begin{equation}\label{equ:d_dop_g}
    g=\mathbb{E}_{\bm\tau\sim\mathcal{D}}\left[{\textstyle\sum}_i k_i(\bm\tau) \nabla_{\theta_{i}} \mu_i(\tau_i; \theta_{i}) \nabla_{a_i} Q_i^{\phi_i}(\bm\tau, a_i)|_{a_i=\mu_i(\tau_i; \theta_{i})}\right].
\end{equation}
Detailed proof can be found in Appendix~\ref{appx:d_dop_pg}. Similar to the stochastic case, This result reveals that updates of individual deterministic policies depend on local critics when a linearly decomposed critic is used. Based on Eq.~\ref{equ:d_dpg_c} and Eq.~\ref{equ:d_dop_g}, we develop the DOP algorithm for learning deterministic policies in continuous action spaces, which is described in Appendix~\ref{appx:algorithm} and called \emph{deterministic \name}.


\subsubsection{Representation Capacity of Deterministic \name~Critics}\label{sec:d_dop-rcc}

In continuous and smooth environments, we first show that a \name~critic has sufficient expressive capability to represent Q values in the proximity of $\bm\mu(\bm \tau), \forall \bm\tau$ with a bounded error. For simplicity, we denote $O_{\delta}(\bm \tau)=\{\bm a| \parallel \bm a - \bm\mu(\bm\tau) \parallel_2 \leq \delta \}$.

\begin{restatable}{fact}{cas} \label{fact:cas}
Assume that $\forall \bm \tau, \bm a, \bm a' \in O_{\delta}(\bm \tau)$, $\parallel \nabla_{\bm a} Q_{\text{tot}}^{\bm\mu}(\bm \tau, \va) - \nabla_{\bm a'} Q_{\text{tot}}^{\bm\mu}(\bm \tau, \bm a') \parallel_2 \leq L \parallel \bm a - \bm a' \parallel_2$. The estimation error of a \name~critic can be bounded by $O(L\delta^2)$ for $\bm a \in O_{\delta}(\bm \tau), \forall \bm \tau$. 
\end{restatable}

Detailed proof can be found in Appendix~\ref{appx:d_dop_pit}. Here we assume that the gradients of Q-values with respect to actions are Lipschitz smooth under the deterministic policy $\bm\mu$. This assumption is reasonable given that Q-values of most continuous environments with continuous policies are rather smooth.

We further show that when Q-values in the proximity of $\bm\mu(\bm \tau), \forall \bm\tau$ are well estimated with a bounded error, deterministic \name~policy gradients are good approximation to the true gradients (Eq.~\ref{equ:maddpg}). Approximately, $|\nabla_{a_i} Q^{\bm \mu}_{tot}(\bm \tau, \bm a) - \nabla_{a_i} k_i(\bm\tau) \nabla_{a_i} Q_i^{\phi_i}(\bm\tau, a_i)| \sim O(L\delta), \forall i$ when $\delta \ll 1$. For detailed proof, we refer readers to Appendix~\ref{appx:d_dop_pit}.

\section{Experiments}\label{sec:exp}
We design experiments to answer the following questions: (1) Does the CDM issue commonly exist and can decomposed critics attenuate it? (Sec.~\ref{sec:exp-cdm},~\ref{sec:exp-sc2-ablation}, and~\ref{sec:exp-mpe}) (2) Can our decomposed multi-agent tree backup algorithm improve the efficiency of off-policy learning? (Sec.~\ref{sec:exp-sc2-ablation}) (3) Can deterministic \name~learn reasonable credit assignment? (Sec.~\ref{sec:exp-mpe}) (4) Can \name~outperform state-of-the-art MARL algorithms? For evaluation, all the results are averaged over $12$ different random seeds and are shown with $95\%$ confidence intervals.

\subsection{Didactic Example: The CDM Issue and Bias-Variance Trade-Off}\label{sec:exp-cdm}
\begin{figure}
    \centering
    \includegraphics[width=\linewidth]{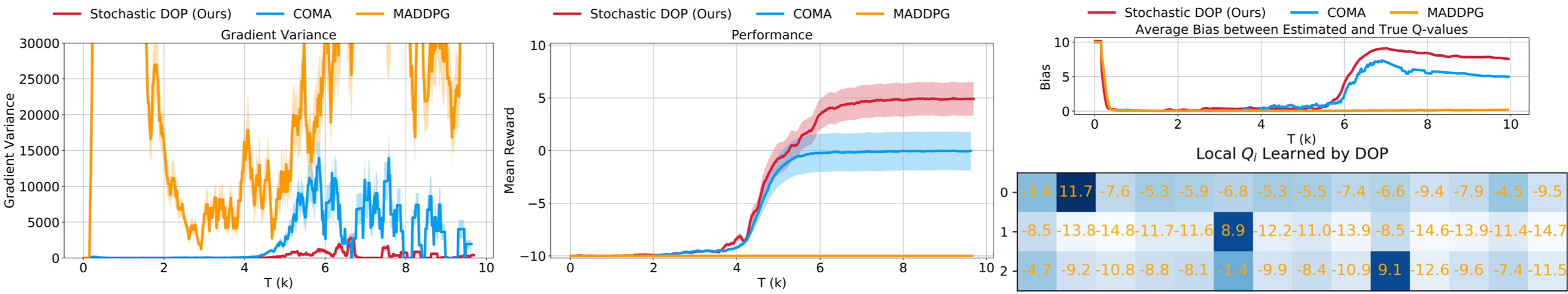}
    \caption{Bias-variance trade-off of \name~on the didactic example. Left: gradient variance; Middle: Performance; Right: Average bias in Q estimations; Right-bottom: the element in $i$th row and $j$th column is the local Q value learned by \name~for agent $i$ taking action $j$.}
    \label{fig:lc_tradeoff}
\end{figure}


We use a didactic example to demonstrate how \name~attenuates CDM and achieves bias-variance trade-off. In a state-less game with $3$ agents and $14$ actions, if agents take $\mathtt{action 1,5,9}$, respectively, they get a team reward of $10$; otherwise $\shortn 10$. We train stochastic \name, COMA, and MADDPG for $10K$ timesteps and show the gradient variance, value estimation bias, and learning curves in Fig.~\ref{fig:lc_tradeoff}. Gumbel-Softmax trick~\citep{jang2017categorical, maddison2017concrete} is used to enable MADDPG to learn in discrete action spaces. 

Fig.~\ref{fig:lc_tradeoff}-right shows the average bias in the estimations of all Q values. We see that linear decomposition introduces extra estimation errors. However, the variance of \name~policy gradients is much smaller than other algorithms (Fig.~\ref{fig:lc_tradeoff}-left). As discussed in Sec.~\ref{sec:ana-CDM}, large variance of other algorithms is due to the CDM issue that undecomposed joint critics are affected by actions of all agents. Free from the influence of other agents, \name~preserves the order of local Q-values (bottom of Fig.~\ref{fig:lc_tradeoff}-right) and effectively reduces the variance of policy gradients. In this way, \name~sacrifices value estimation accuracy for accurate and low-variance policy gradients, which explains why it can outperform other algorithms (Fig.~\ref{fig:lc_tradeoff}-middle).



\begin{figure}
    \centering
    \includegraphics[width=\linewidth]{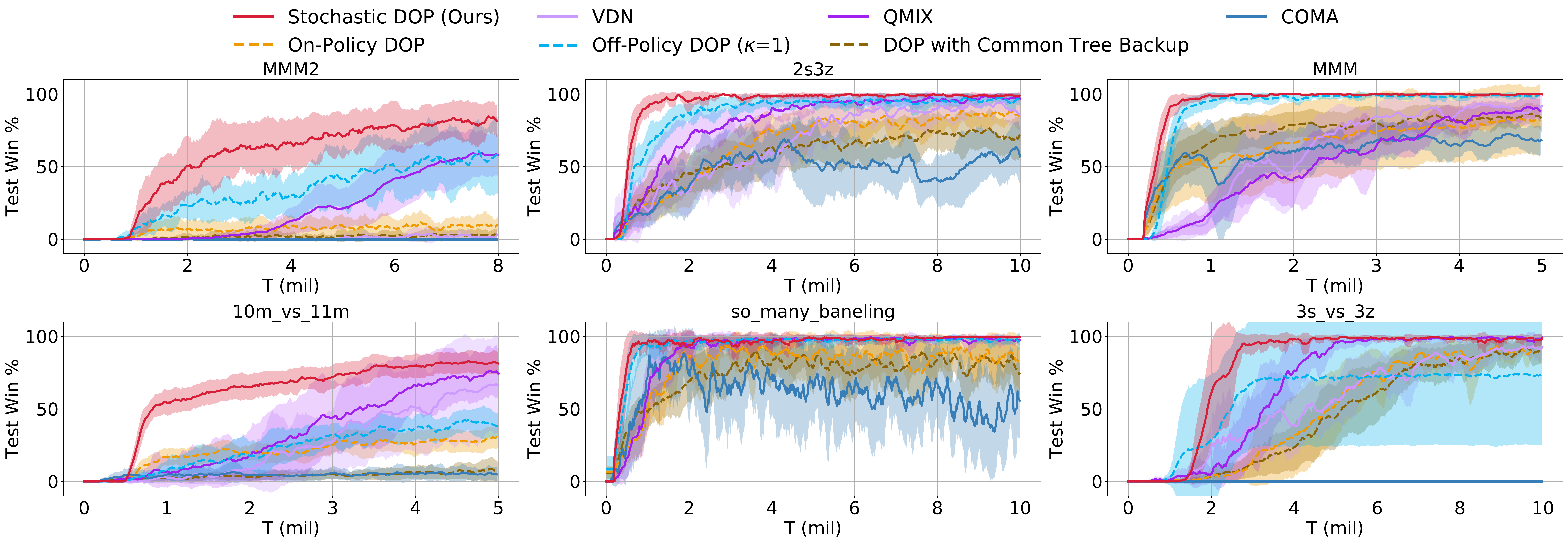}
    \caption{Comparisons with baselines and ablations on the SMAC benchmark.}
    \label{fig:lc_sc2}
\end{figure}
\subsection{Discrete Action Spaces: The StarCraft II Micromanagement Benchmark}\label{sec:exp-sc2}
We evaluate stochastic \name~on the challenging SMAC benchmark~\citep{samvelyan2019starcraft} for its high control complexity. We compare our method with state-of-the-art multi-agent stochastic policy gradient method (COMA) and value-based methods (VDN and QMIX). For stochastic \name, we fix the hyperparameter setting and network structure in all experiments which are described in Appendix~\ref{appx:architecture}. For baselines, we use their default hyperparameter settings that have been fine-tuned on the SMAC benchmark. Results are shown in Fig.~\ref{fig:lc_sc2}. Stochastic \name~significantly outperforms all the baselines by a wide margin. To our best knowledge, this is the first time that a MAPG method has significantly better performance than state-of-the-art value-based methods. 

\subsubsection{Ablations}\label{sec:exp-sc2-ablation}
Stochastic \name~has three main components: (a) off-policy policy evaluations, (b) the decomposed critic, and (c) decomposed multi-agent tree backup. By design, component (a) improves sample efficiency, component (b) can attenuate the CDM issue, and component (c) makes off-policy policy evaluations tractable. We test the contribution of each component by carrying out the following ablation studies.

\textbf{Off-Policy Learning} In our method, $\kappa$ controls the "off-policyness" of training. For \name, we set $\kappa$ to $0.5$. To demonstrate the effect of off-policy learning, we change $\kappa$ to $0$ and $1$ and compare the performance. In Fig.~\ref{fig:lc_sc2}, we can see that both \name~and off-policy \name~perform much better than the on-policy version ($\kappa$=0), highlighting the importance of using off-policy data. Moreover, purely off-policy learning generally needs more samples to achieve similar performance to \name. Mixing with on-policy data can largely improve training efficiency.

\textbf{The CDM Issue} \emph{On-Policy \name}~uses the same decomposed critic structure as \name, but is trained only with on-policy data and does not use tree backup. The only difference between \emph{On-Policy \name}~and COMA is that the former one uses a decomposed joint critic. Therefore, given that a COMA critic has a more powerful expression capacity than a \name~critic, the outperformance of \emph{On-Policy \name}~against COMA shows the effect of CDM. COMA is not stable and may diverge after a near-optimal policy has been learned. For example, on map $\mathtt{so\_many\_baneling}$, COMA policies degenerate after 2M steps. In contrast, On-Policy \name~can converge with efficiency and stability.

\textbf{Decomposed Multi-Agent Tree Backup} \emph{\name~with Common Tree Backup} (\name~without component (c)) is the same as \name~except that $\mathbb{E}_{\bm\pi}[Q_{tot}^\phi(\bm\tau, \cdot)]$ is estimated by sampling 200 joint actions from $\bm\pi$. Here, we estimate this expectation by sampling because direct computation is intractable (for example, $20^{10}$ summations are needed on the map $\mathtt{MMM}$). Fig.~\ref{fig:lc_sc2} shows that when the number of agents increases, sampling becomes less efficient, and common tree backup performs even worse than \emph{On-Policy \name}. In contrast, \name~with decomposed tree backup can quickly and stably converge using a similar number of summations. 


\begin{figure}
    \centering
    \includegraphics[height=2.3cm]{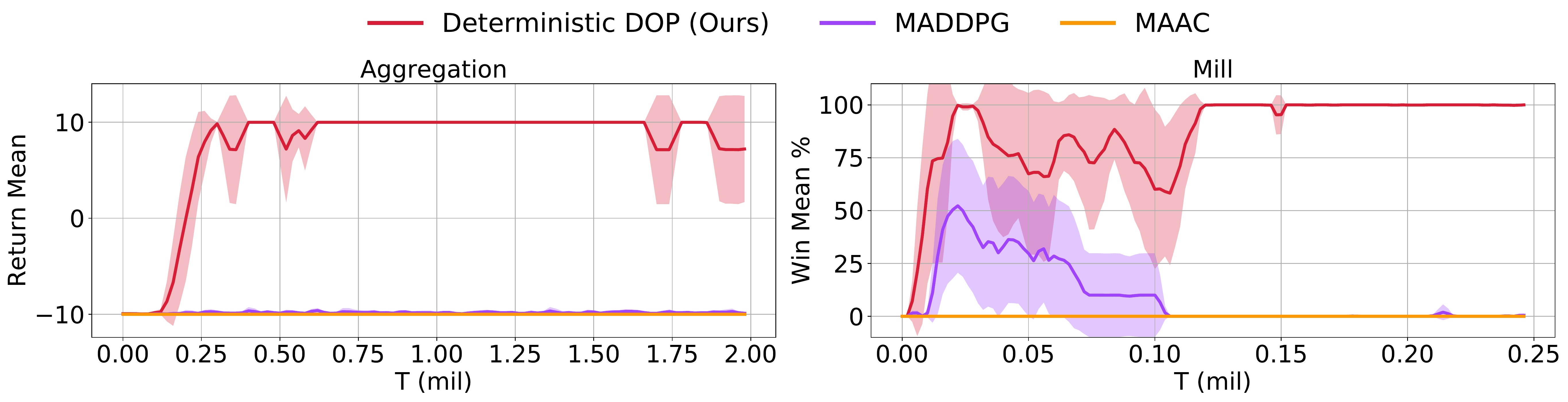}
    \includegraphics[height=2.3cm]{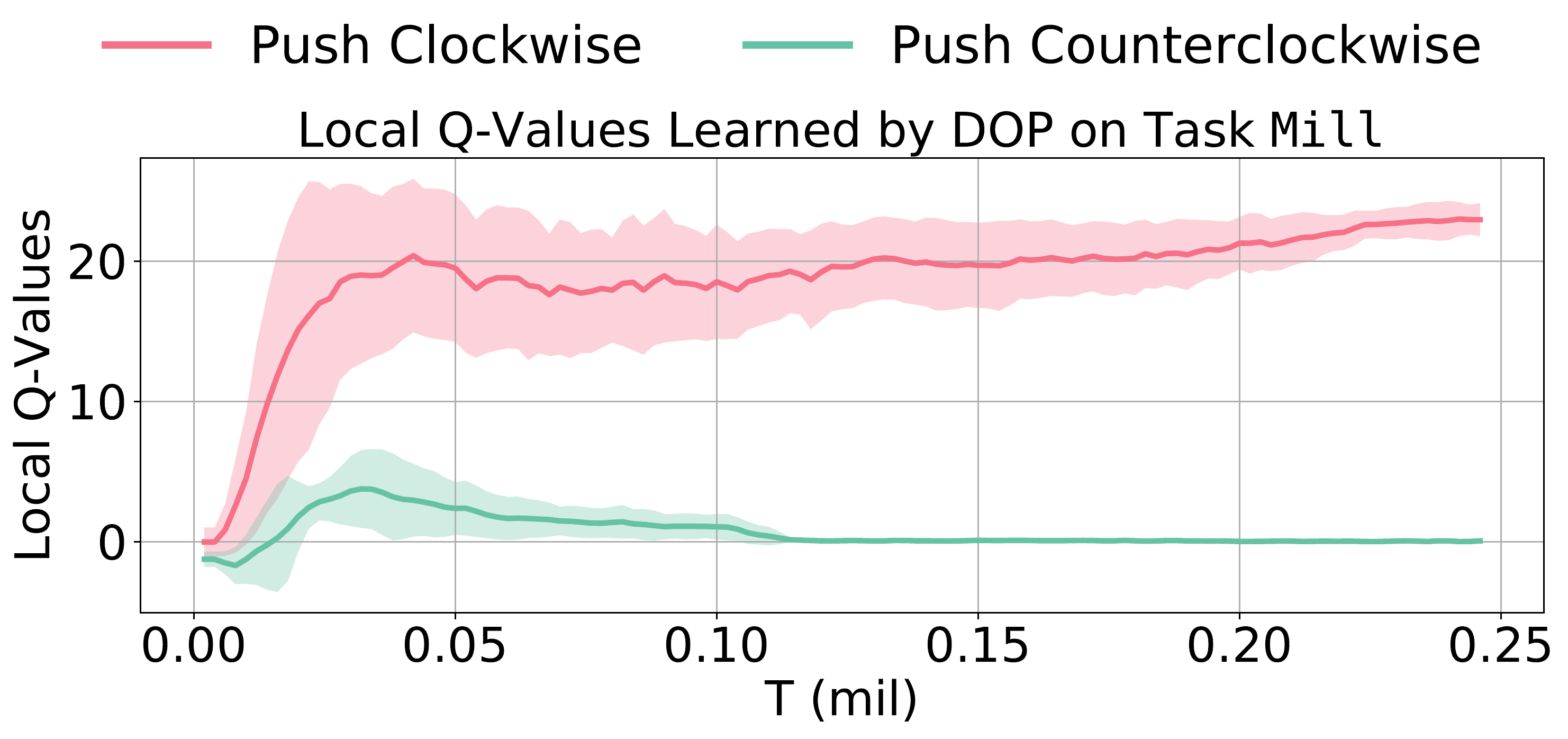}
    \caption{Left and middle: performance comparisons with COMA and MAAC on MPE. Right: The learned credit assignment mechanism on task $\mathtt{Mill}$ by deterministic \name.}
    \label{fig:mpe}
\end{figure}
\subsection{Continuous Action Spaces: Multi-Agent Particle Environments}\label{sec:exp-mpe}
We evaluate deterministic \name~on multi-agent particle environments (MPE, \citep{mordatch2018emergence}), where agents take continuous actions in continuous spaces. We compare our method with MADDPG~\citep{lowe2017multi} and MAAC~\citep{iqbal2019actor}. Hyperparameters and the network structure are fixed for deterministic \name~across experiments, which are described in Appendix~\ref{appx:architecture}.

\textbf{The CDM Issue} We use task $\mathtt{Aggregation}$ as an example to show that deterministic \name~attenuates the CDM issue. In this task, $5$ agents navigate to one landmark. Only when all agents reach the landmark will they get a team reward of $10$ and successfully end the episode; otherwise, an episode ends after $25$ timesteps and agents get a reward of $-10$. $\mathtt{Aggregation}$ is a typical example where other agents' actions can influence an agent's local policy through an undecomposed joint critic. Intuitively, as long as one agent does not reach the landmark, the centralized Q value is negative, confusing other agents who get to the landmark. This intuition is supported by the empirical results shown in Fig.~\ref{fig:mpe}-left -- methods with undecomposed critics can find rewarding configurations but then quickly diverge, while \name~converges with stability. 

\textbf{Credit Assignment} We use task $\mathtt{Mill}$ to show that \name~can learn effective credit assignment mechanisms. In this task, $10$ agents need to rotate a millstone clockwise. They can push the millstone clockwise or counterclockwise with force between $0$ and $1$. If the millstone's angular velocity, $\omega$, gets greater than $30$, agents are rewarded $3$ per step. If $\omega$ exceeds $100$ in $10$ steps, the agents win the episode and get a reward of $10$; otherwise, they lose and get a punishment of -$10$. Fig.~\ref{fig:mpe}-right shows that deterministic \name~can gradually learn a reasonable credit assignment during training, where rotating the millstone clockwise has much larger Q-values. This explains why deterministic \name~outperforms previous state-of-the-art deterministic MAPG methods, as shown in Fig.~\ref{fig:mpe}-middle.

\section{Closing Remarks}
This paper pinpointed drawbacks that hinder the performance of state-of-the-art MAPG algorithms: on-policy learning of stochastic policy gradient methods, the centralized-decentralized mismatch problem, and the credit assignment issue in deterministic policy learning. We proposed decomposed actor-critic methods (\name) to address these problems. Theoretical analyses and empirical evaluations demonstrate that \name~can achieve stable and efficient multi-agent off-policy learning.




\bibliography{iclr2021_conference}
\bibliographystyle{iclr2021_conference}

\newpage
\appendix
\section{Mathematical details for stochastic \name}

\subsection{Decomposed critics enable tractable multi-agent tree backup}
In Sec.~\ref{sec:s_dop_tb}, we propose to use tree backup~\citep{precup2000eligibility, munos2016safe} to carry out multi-agent off-policy policy evaluation. When a joint critic is used, calculating $\mathbb{E}_{\bm\pi} \left[Q^{\phi}_{tot}(\bm\tau,\cdot)\right]$ requires $O(|A|^n)$ steps of summation. To solve this problem, \name~uses a linearly decomposed critic, and it follows that:
\begin{equation}
    \begin{aligned}
    \mathbb{E}_{\bm\pi}[Q_{tot}^\phi(\bm\tau, \va)] & = \sum_{\va} \bm\pi(\va|\bm\tau) Q_{tot}^\phi(\bm\tau, \va) = \sum_{\va} \bm\pi(\va|\bm\tau) \left[\sum_i k_i(\bm\tau)Q_i^{\phi_i}(\bm\tau, a_i) + b(\bm\tau)\right]\\
    & = \sum_{\va} \bm\pi(\va|\bm\tau) \sum_i k_i(\bm\tau)Q_i^{\phi_i}(\bm\tau, a_i) + \sum_{\va} \bm\pi(\va|\bm\tau) b(\bm\tau)  \\
    & = \sum_i \sum_{a_i} \pi_i(a_i|\tau_i) k_i(\bm\tau)Q_i^{\phi_i}(\bm\tau, a_i) \sum_{\va_{\shortn i}}\bm\pi_{\shortn i}(\va_{\shortn i}|{\bm\tau}_{\shortn i}) + b(\bm\tau) \\
    & = \sum_i k_i(\bm\tau)\mathbb{E}_{\pi_i}[Q_i^{\phi_i}(\bm\tau, \cdot)] + b(\bm\tau),
    \end{aligned}
\end{equation}
which means the complexity of calculating this expectation is reduced to $O(n|A|)$.

\subsection{Stochastic \name~policy gradients}\label{appx:stochastic_pg}
\subsubsection{On-policy version}
In Sec.~\ref{sec:s_dop_tb}, we give the on-policy stochastic \name~policy gradients:
\begin{equation}
    g=\mathbb{E}_{\bm\pi}\left[{\textstyle\sum}_i k_i(\bm\tau) \nabla_{\theta_{i}}\log{\pi_i(a_{i}|\tau_{i}; \theta_i)}Q^{\phi_i}_{i}(\bm\tau,a_{i})\right].
\end{equation}
We now derive it in detail.
\begin{proof}
We use the aristocrat utility~\citep{wolpert2002optimal} to perform credit assignment:

\begin{align*}
U_i(\bm\tau, a_i) &=Q^{\phi}_{tot}(\bm\tau,\va)-\sum_{x}{\pi_{i}(x|\tau_{i})Q^{\phi}_{tot}(\bm\tau,(x,\va_{-i}))}\\
&=\sum_{j}{k_j(\bm\tau)Q^{\phi_j}_{j}(\bm\tau,a_j)}-\sum_{x}{\pi_{i}(x|\tau_{i})\left[\sum_{j\neq i}k_{j}(\bm\tau)Q^{\phi_j}_{j}(\bm\tau,a_j)+k_{i}(\bm\tau)Q^{\phi_i}_{i}(\bm\tau,x)\right]}\\
&=k_i(\bm\tau)Q^{\phi_i}_{i}(\bm\tau,a_{i})-k_{i}(\bm\tau)\sum_{x}{\pi_{i}(x|\tau_{i})Q^{\phi_i}_{i}(\bm\tau,x)}\\
&=k_{i}(\bm\tau)\left[Q^{\phi_i}_{i}(\bm\tau,a_{i}) - \sum_{x}{\pi_{i}(x|\tau_{i})Q^{\phi_i}_{i}(\bm\tau,x)}\right],
\end{align*}
It is worth noting that $U_i$ is independent of other agents' actions. Then, for the policy gradients, we have:

\begin{align*}
g&=\mathbb{E}_{\bm\pi}[\sum_{i}{\nabla_{\theta}\log\pi_{i}(a_{i}|\tau_{i})U_{i}(\bm\tau, a_i)}]\\
&=\mathbb{E}_{\bm\pi}\left[\sum_{i}{\nabla_{\theta}\log\pi_{i}(a_{i}|\tau_{i})k_{i}(\bm\tau)\left(Q^{\phi_i}_{i}(\bm\tau,a_{i}) - \sum_{x}{\pi_{i}(x|\tau_{i})Q^{\phi_i}_{i}(\bm\tau,x)}\right)}\right]\\
&=\mathbb{E}_{\bm\pi}\left[\sum_{i}{\nabla_{\theta}\log\pi_{i}(a_{i}|\tau_{i})k_{i}(\bm\tau)Q^{\phi_i}_{i}(\bm\tau,a_{i})}\right].
\end{align*}
\end{proof}

\subsubsection{Off-policy version}
In Appendix~\ref{appx:stochastic_pg}, we derive the on-policy policy gradients for updating stochastic multi-agent policies. Similar to policy evaluation, using off-policy data can improve the sample efficiency with regard to policy improvement.

Using the linearly decomposed critic architecture, the off-policy policy gradients for learning stochastic policies are:
\begin{equation}\label{equ:s_dop-oppg}
    g=\mathbb{E}_{\bm\beta}\left[\frac{\pi(\bm\tau,\va)}{\beta(\bm\tau,\va)}{\textstyle\sum}_i k_i(\bm\tau) \nabla_{\theta}\log{\pi_i(a_{i}|\tau_{i}; \theta_i)}Q^{\phi_i}_{i}(\bm\tau,a_{i})\right].
\end{equation}

\begin{proof}
The objective function is:
\begin{align*}
J(\theta)&=\mathbb{E}_{\bm\beta}\left[V_{tot}^{\bm\pi}(\bm\tau)\right].
\end{align*}

Similar to~\cite{degris2012off}, we have:
\begin{align*}
\nabla_{\theta} J(\theta)&=\mathbb{E}_{\bm\beta}\left[\frac{\pi(\va|\bm\tau)}{\beta(\va|\bm\tau)}\sum_{i}{\nabla_{\theta}\log\pi_{i}(a_{i}|\tau_{i})U_{i}(\bm\tau,a_i)}\right]\\
&=\mathbb{E}_{\bm\beta}\left[\frac{\pi(\va|\bm\tau)}{\beta(\va|\bm\tau)}\sum_{i}{\nabla_{\theta}\log\pi_{i}(a_{i}|\tau_{i})k_{i}(\bm\tau)A_{i}(\bm\tau,a_{i})}\right]\\
&=\mathbb{E}_{\bm\beta}\left[\frac{\pi(\va|\bm\tau)}{\beta(\va|\bm\tau)}\sum_{i}{\nabla_{\theta}\log\pi_{i}(a_{i}|\tau_{i})k_{i}(\bm\tau)Q^{\phi_i}_{i}(\bm\tau,a_{i})}\right].
\end{align*}

\end{proof}

\subsection{The CDM issue}
In Sec.\ref{sec:exp-cdm}, we empirically show that stochastic \name~can reduce variance in policy gradients. We now discuss the reason from a theoretical perspective. 

Denote r.v.s $g_{1} = \nabla_{\theta_{i}}\log{\pi_i(a_{i}|\tau_{i};\theta_i)}Q^{\phi}_{tot}(\bm\tau, \mathbf{a})$, $g_{2}=k_i(\bm\tau) \nabla_{\theta_{i}}\log{\pi_i(a_{i}|\tau_{i};\theta_i)}$ $Q^{\phi_i}_{i}(\bm\tau,a_{i})$. We assume that the gradient of $\pi_i$ with respect to $\theta_i$ is bounded: $\nabla_{\theta_{i}}\log{\pi_i(a_{i}|\tau_{i};\theta_i)}\in[L,R]$. Let $X_{i}=k_{i}(\bm\tau)Q_{i}^{\phi_i}(\bm\tau,a_{i})$, and assume $X_1, X_2, \dots,X_n$ are i.i.d. random variables with mean $\mu$ and variance $\sigma^2$. It follows that:
\begin{align*}
\frac{\Var_{\pi_i}(g_2)}{\Var_{\bm\pi}(g_1)}&=\frac{\Var_{\bm\pi}(g_2)}{\Var_{\bm\pi}(g_1)}\\
&=\frac{\Var_{\bm\pi}(k_i(\bm\tau) \nabla_{\theta_{i}}\log{\pi_i(a_{i}|\tau_{i};\theta_i)}Q^{\phi_i}_{i}(\bm\tau,a_{i}))}{\Var_{\bm\pi}(\sum_{j=1}^{n}k_j(\bm\tau) \nabla_{\theta_{i}}\log{\pi_i(a_{i}|\tau_{i};\theta_i)}Q^{\phi_j}_{j}(\bm\tau,a_{j}))}\\
&\le\frac{R^2\Var_{\bm\pi}(k_i(\bm\tau) Q^{\phi_i}_{i}(\bm\tau,a_{i}))+(R-L)\mathbb{E}_{\bm\pi}[k_i(\bm\tau) Q^{\phi_i}_{i}(\bm\tau,a_{i})]^2}{\Var_{\bm\pi}(\sum_{j=1}^{n}k_j(\bm\tau) \nabla_{\theta_{i}}\log{\pi_i(a_{i}|\tau_{i};\theta_i)}Q^{\phi_j}_{j}(\bm\tau,a_{j}))}\\
&\le\frac{R^2\Var_{\bm\pi}(k_i(\bm\tau) Q^{\phi_i}_{i}(\bm\tau,a_{i}))+(R-L)\mathbb{E}_{\bm\pi}[k_i(\bm\tau Q^{\phi_i}_{i}(\bm\tau,a_{i})]^2}{L^2\Var_{\bm\pi}(\sum_{j=1}^{n}k_j(\bm\tau) Q^{\phi_j}_{j}(\bm\tau,a_{j}))-(R-L)\mathbb{E}_{\bm\pi}[\sum_{j=1}^{n}k_j(\bm\tau) Q^{\phi_j}_{j}(\bm\tau,a_{j})]^2}\\
&=\frac{R^2\sigma^2+(R-L)\mu^2}{nL^2\sigma^2-n^2(R-L)\mu^2}=O(\frac{1}{n}).
\end{align*}

This means that for any $\bm\tau$, under the stated assumptions, we have $\frac{\Var_{\pi_i}(g_2)}{\Var_{\bm\pi}(g_1)}=O(\frac{1}{n})$. However, these assumptions are quite strong.
\section{Mathematical details for deterministic \name}
\subsection{Deterministic \name~policy gradient theorem}\label{appx:d_dop_pg}

In Sec.~\ref{sec:d_dop_pg}, we give the following deterministic \name~policy gradients:
\begin{equation}
    \nabla J(\theta)=\mathbb{E}_{\bm\tau\sim\mathcal{D}}\left[{\textstyle\sum}_i k_i(\bm\tau) \nabla_{\theta_{i}} \mu_i(\tau_i; \theta_{i}) \nabla_{a_i} Q_i^{\phi_i}(\bm\tau, a_i)|_{a_i=\mu_i(\tau_i; \theta_{i})}\right].
\end{equation}
Now we present the derivation of this update rule.

\begin{proof}
Drawing inspirations from single-agent cases~\citep{silver2014deterministic}, we have:
\begin{align*}
\nabla J(\theta)&=\mathbb{E}_{\bm\tau\sim \mathcal{D}}[\nabla_{\theta}Q^{\phi}_{tot}(\bm\tau,\va)]\\
&=\mathbb{E}_{\bm\tau\sim \mathcal{D}}[\sum_{i}{\nabla_{\theta}k_i(\bm\tau)Q_{i}^{\phi_i}(\bm\tau,a_i)|_{a_i=\mu_i(\tau_i;\theta_i)}}]\\
&=\mathbb{E}_{\bm\tau\sim \mathcal{D}}[\sum_{i}{\nabla_{\theta}\mu_i(\tau_i;\theta_i)\nabla_{a_i}k_i(\bm\tau)Q_{i}^{\phi_i}(\bm\tau,a_i)|_{a_i=\mu_i(\tau_i;\theta_i)}}].
\end{align*}
\end{proof}

\section{Proof of stochastic \name~policy improvement theorem}\label{appx:s_dop_pit}

Inspired by previous work~\citep{degris2012off}, we relax the requirement that $Q_{tot}^{\phi}$ is a good estimate $Q_{tot}^{\pi}$ and show that stochastic \name~still guarantees policy improvement. First, we define:
\begin{equation*} 
    Q_i^{\pi}(\bm \tau, a_i) = \sum_{\bm  a_{\shortn i}} \bm\pi_{\shortn i}(\bm a_{\shortn i} | \bm \tau_{\shortn i}) Q_{tot}^{\bm\pi}(\bm \tau, \va), \quad A_i^{\bm\pi}(\bm \tau, a_i) = \sum_{\bm a_{\shortn i}} \bm\pi(\bm a_{\shortn i} | \bm \tau_{\shortn i}) A_i^{\bm\pi}(\bm \tau, \va).
\end{equation*}
To analyze which critic's estimation can minimize the TD-error is challenging. To make it tractable, some works~\citep{feinberg2018model} simplify this process as an MSE problem. In stochastic \name, we adopt the same technique and regard critic's learning as the following MSE problem: 
\begin{align} \label{appx:equ:mse}
    L(\phi) = \sum_{\va, \bm\tau} p(\bm\tau) \pi(\va | \bm\tau)\left(Q^{\bm\pi}_{tot}(\bm\tau, \va) - Q^{\phi}_{tot}(\bm\tau, \va)) \right)^2
\end{align}

where $Q^{\bm\pi}_{tot}(\bm\tau, \va)$ are the true values, which are fixed during optimization. In the following lemma, we show that monotonic decomposition can preserve the order of local action values. Without loss of generality, we will consider a given $\bm\tau$. 

\begin{lemma} \label{lemma::mono}
We consider the following optimization problem:
\begin{align} 
    L_{\bm\tau}(\phi) = \sum_{\va} \bm\pi(\va | \bm\tau)\left(Q^{\bm\pi}(\bm\tau, \va) - f(\mathbf{Q}^{\bm\phi}(\bm\tau, \va)) \right)^2.
\end{align}
Here, $f(\mathbf{Q}^{\bm\phi}(\bm\tau, \va)): \mathcal{R}^n \to \mathcal{R}$, and $\mathbf{Q}^{\bm\phi}(\bm\tau, \va)$ is a vector with the $i^{th}$ entry being $Q^{\phi_i}_i(\bm\tau, a_i)$. $f$ satisfies that $\frac{\partial f}{\partial Q^{\phi_i}_i (\bm\tau, a_i)} > 0$ for any $i, a_i$.

Then, for any local optimal solution, it holds that:
\begin{equation*}
    Q^{\bm\pi}_i(\bm\tau, a_i) \geq Q^{\bm\pi}_i(\bm\tau, a_i') \iff Q^{\phi_i}_i (\bm\tau, a_i) \geq Q^{\phi_i}_i (\bm\tau, a_i'), \ \ \ \ \ \forall i, a_i, a_i'.
\end{equation*}
\end{lemma}

\begin{proof}
A necessary condition for a local optimal is:
\begin{align*}
    \frac{\partial L_{\bm\tau}(\phi)}{\partial Q^{\phi_i}_i(\bm\tau, a_i)} & = \pi_i(a_i | \tau_i) \sum_{a_{\shortn i}} \prod_{j\ne i} \pi_j(a_j | \tau_j) \left(Q_{tot}^{\bm\pi}(\bm\tau, \va) - f(\mathbf{Q}^{\bm\phi}(\bm\tau, \va))\right) (-\frac{\partial f}{\partial Q^{\phi_i}_i (\bm\tau, a_i)}) = 0, \ \ \ \forall i, a_i.
\end{align*}

This implies that, for $\forall i, a_i$, we have
\begin{align*}
 & \sum_{a_{\shortn i}} \prod_{j\ne i} \pi_j(a_j | \tau_j) (Q^{\bm\pi}_{tot}(\bm\tau, \va) - f(\mathbf{Q}^{\bm\phi}(\bm\tau, \va))) = 0 \\
 \Rightarrow & \sum_{a_{\shortn i}} \bm\pi_{\shortn i}(\va_{\shortn i} | \bm\tau_{\shortn i}) f(\mathbf{Q}^{\bm\phi}(\bm\tau, (a_i, \va_{-i}))) = Q^{\bm\pi}_i(\bm\tau, a_i)
\end{align*}

We consider the function $q(\bm\tau, a_i) = \sum_{\va_{\shortn i}} \bm\pi_{\shortn i}(\va_{\shortn i} | \bm\tau_{-i}) f(\mathbf{Q}^{\bm\phi}(\bm\tau, (a_i, \va_{\shortn i})))$, which is a function of $\mathbf{Q}^{\phi}$. Its partial derivative with respect to $Q^{\phi_i}_i(\bm\tau, a_i)$ is:
\begin{equation*}
    \frac{\partial q(\bm\tau, a_i)}{\partial Q^{\phi_i}_i (\bm\tau, a_i)} = \sum_{\va_{\shortn i}} \bm\pi_{\shortn i}(\va_{\shortn i} | \bm\tau_{\shortn i}) \frac{\partial f(\mathbf{Q}^{\bm\phi}(\bm\tau, (a_i, \va_{\shortn i})))}{\partial Q^{\phi_i}_i(\bm\tau, a_i)} > 0
\end{equation*}

Therefore, if $Q_i^{\bm\pi}(\bm\tau, a_i) \geq Q_i^{\bm\pi}(\bm\tau, a_i')$, then any local minimal of $L_{\bm\tau}(\phi)$ satisfies $Q^{\phi_i}_i (\bm\tau, a_i) \geq Q^{\phi_i}_i (\bm\tau, a_i')$.
\end{proof} 

In our linearly decomposed critic architecture, $k_i(\bm\tau) > 0, \forall i$, which satisfies the condition $\frac{\partial f}{\partial Q^{\phi_i}_i (\bm\tau, a_i)} > 0$. Therefore, Fact~\ref{fact:mono} holds as a corollary of Lemma~\ref{lemma::mono}:

\mono*


Based on this conclusion, we are able to prove the policy improvement theorem for stochastic \name. It shows that even without an accurate estimate of $Q_{tot}^{\bm\pi}$, the stochastic \name~policy updates can still improve the objective function $J(\bm\pi) = \mathbb{E}_{\bm\pi} [\sum_t \gamma^t r_t]$. We first prove the following lemma.

\begin{restatable}{lemma}{lemmasit} \label{fact:sit}
For two sequences $\{a_i\}, \{b_i\}, i \in [n]$ listed in an increasing order. If $\sum_i b_i = 0$, then $\sum_i a_ib_i \geq 0$.
\end{restatable}
 
 \begin{proof}
We denote $\bar{a} = \frac{1}{n} \sum_i a_i$, then $\sum_i a_ib_i = \bar{a}(\sum_i b_i) + \sum_i \tilde{a}_i b_i$ where $\sum_i \tilde{a}_i = 0$. Without loss of generality, we assume that $\bar{a}_i = 0, \forall i$. $j$ and $k$ which $a_j \leq 0, a_{j+1} \geq 0$ and $b_k \leq 0, b_{k+1} \geq 0$. Since $a, b$ are symmetric, we assume $j \leq k$. Then we have 
\begin{align*}
    \sum_{i \in [n]} a_ib_i & = \sum_{i \in [1, j]} a_ib_i + \sum_{i \in [j+1, k]} a_ib_i + \sum_{i \in [k+1, n]} a_ib_i\\
    & \geq \sum_{i \in [j+1, k]} a_ib_i + \sum_{i \in [k+1, n]} a_ib_i\\
    & \geq a_k \sum_{i \in [i+1, k]} b_i + a_{k+1}\sum_{i \in [k+1, n]} b_i 
\end{align*}
As $\sum_{i \in [j+1, n]} b_i \geq 0$, we have $- \sum_{i \in [j+1, k]} b_i \leq \sum_{i \in [k+1, n]} b_i$.

Thus, $\sum_{i \in [n]} a_ib_i \geq (a_{k+1} - a_k) \sum_{i \in [k+1, n]} b_i \geq 0$.
 \end{proof}
 
We now prove the policy improvement theorem for stochastic DOP. We restate this theorem as follows for clarity.
\sdoppit*
 
\begin{proof}
Under Fact $\ref{fact:mono}$, it follows that
 \begin{equation}
     Q_i^{\bm\pi^o}(\bm \tau, a_i) > Q_i^{\bm\pi^o}(\bm \tau, a_i') \iff \beta_{a_i, \bm\tau} \geq \beta_{a_i', \bm\tau}.
 \end{equation}
 Since $J(\bm\pi) = \sum_{\bm \tau_0}p(\bm \tau_0) V_{tot}^{\bm\pi}(\bm \tau_0)$, it suffices to prove that $\forall \bm\tau_t, V_{tot}^{\bm\pi}(\bm \tau_t) \geq V_{tot}^{\bm\pi^o}(\bm \tau_t)$. We have:
 \begin{align}
     \sum_{\va_t} \bm\pi(\va_t | \bm\tau_t) Q^{\bm\pi^o}_{tot}(\bm \tau_t, \va_t) & = \sum_{\va_t} \left(\prod_{i = 1}^n \pi_i(a_i^t | \tau_i^t) \right) Q_{tot}^{\bm\pi^o}(\bm \tau_t, \va_t) \nonumber \\
     & = \sum_{\va_t} \left( \prod_{i = 1}^n (\pi^o_i(a_i^t | \tau_i^t) + \beta_{a_i^t, \bm\tau_t} \delta) \right) Q_{tot}^{\bm\pi^o}(\bm \tau_t, \va_t) \nonumber \\
     & = V^{\bm\pi^o}_{tot}(\bm\tau_t) + \delta \sum_{i = 1}^n \sum_{\va_t} \beta_{a_i^t, \bm\tau_t} \left( \prod_{j \neq i} \pi_j^o(a_j^t | \tau_j^t) \right) Q_{tot}^{\bm\pi^o}(\bm \tau_t, \va_t) \nonumber + o(\delta)\\
     & = V^{\bm\pi^o}_{tot}(\bm\tau_t) + \delta \sum_{i = 1}^n \sum_{a_i^t} \beta_{a_i^t, \bm\tau_t} Q^{\bm\pi^o}_i(\bm\tau_t, a_i^t) + o(\delta). \label{a-equ:tmpl}
 \end{align}
 
 Since $\delta$ is sufficiently small, in the following analysis we omit $o(\delta)$. Observing that $\sum_{a_i} \pi_i(a_i | \tau_i) = 1, \forall i$, we get $\sum_{a_i} \beta_{a_i, \bm\tau} = 0$. Thus, by Lemma~\ref{fact:sit} and Eq.~\ref{a-equ:tmpl}, we have
 \begin{equation}
     \sum_{\va_t} \bm\pi(\va_t | \bm\tau_t) Q_{tot}^{\bm\pi^o}(\bm\tau_t, \va_t) \geq V_{tot}^{\bm\pi^o}(\bm\tau_t).
 \end{equation}
 
Similar to the policy improvement theorem for tabular MDPs \citep{sutton2018reinforcement} , we have
\begin{align*}
    V^{\bm\pi^o}_{tot}(\bm\tau_t) & \leq \sum_{\bm a_t} \bm\pi(\va_t | \bm\tau_t) Q_{tot}^{\bm\pi^o}(\bm \tau_t, \va_t) \\
    & = \sum_{\va_t} \bm\pi(\va_t | \bm\tau_t) \left( r(\bm\tau_t, \va_t) + \gamma \sum_{\bm \tau_{t+1}} p(\bm \tau_{t+1} | \bm \tau_t, \va_t) V^{\bm\pi^o}_{tot}(\bm \tau_{t+1}) \right) \\
    & \leq \sum_{\va_t} \bm\pi(\va_t | \bm\tau_t) \left( r(\bm \tau_t, \va_t) + \gamma \sum_{\bm \tau_{t+1}} p(\bm \tau_{t+1} | \bm \tau_t, \va_t) \left( \sum_{\va_{t+1}} \bm\pi(\va_{t+1} | \bm\tau_{t+1}) Q^{\bm\pi^o}_{tot}(\bm \tau_{t+1}, \va_{t+1}) \right) \right) \\
    & \leq \cdots \\
    & \leq V^{\bm\pi}_{tot}(\bm \tau_t).
\end{align*}

This implies $J(\bm\pi) \geq J(\bm\pi^o)$ for each update. 

Moreover, we verify that $\forall \tau, a_i', a_i, Q^{\phi_i}_i(\bm \tau, a_i) > Q^{\phi_i}_i(\bm \tau, a_i') \iff \beta_{a_i, \bm\tau} \geq \beta_{a_i', \bm\tau}$ (the \textsc{Monotone} condition) holds for any $\bm\pi$ with a tabular expression. For these $\bm\pi$, let $\pi_i(a_i | \tau_i) = \theta_{a_i, \bm \tau}$, then it holds that $\sum_{a_i} \theta_{a_i, \bm \tau} = 1$.  Since the gradient of policy update can be written as:
\begin{align*}
    \nabla_{\theta} J(\bm\pi_{\theta}) & = \mathbb{E}_{d(\tau)} \left[ \sum_i k_i(\bm\tau) \nabla_{\theta} \log \pi_i(a_i | \tau_i; \theta_i) 
    Q_i^{\phi_i}(\bm \tau, a_i) \right ] \\
    & = \sum_{\tau} d (\tau) \sum_i k_i(\bm\tau)\nabla_{\theta_i} \pi(a_i | \tau_i) Q^{\phi_i}_i(\bm \tau, a_i),
\end{align*}

where $d^{\pi}(\tau)$ is the occupancy measure w.r.t our algorithm. With a tabular expression, the update of each $\theta_{a_i, \bm \tau}$ is proportion to $\beta_{a_i, \bm \tau}$
\begin{equation*}
    \beta_{a_i, \bm \tau} \propto \frac{d \eta(\pi_{\theta})}{d \theta_{a_i, \bm \tau}} = d(\bm \tau) Q_i^{\phi_i}(\bm \tau, a_i) 
\end{equation*}

Clearly, $\beta_{a_i', \bm \tau} \geq \beta_{a_i, \bm \tau} \iff Q_i^{\phi_i}(\bm \tau, a_i') \geq Q_i^{\phi_i}(\bm \tau, a_i)  $ .
\end{proof}

\paragraph{Remark} The above proposition tries to explain why \name~ works well despite the stronger projection error brought in by the decomposition. It is worth mentioning that we usually use neural network to express $\bm\pi$ and Q-values that may violate the \textsc{Monotone} condition. Moreover, minimizing td-error with a target network differs from the MSE problem we assumed in our analysis. However, as long as the above condition holds in most cases, policy improvement can still be guaranteed. As a result, the linearly decomposed critic structure may not provide an accurate estimation of $Q_{tot}^{\bm\pi}$, but it can be used to improve policy $\bm\pi$. Empirically, we have shown that the linearly decomposed critic structure has considerable stable performance despite some potential shortcomings.

\section{Representational capability of deterministic \name~critics}\label{appx:d_dop_pit}

In Sec.~\ref{sec:d_dop-rcc}, we present the following facts about deterministic \name:
\cas*

We consider the Taylor expansion with Lagrange remainder of $Q^{\bm\mu}_{tot}(\bm \tau, \va)$. Namely,
\begin{align*}
Q^{\bm\mu}_{\text{tot}}(\bm \tau, \bm a) = Q^{\bm\mu}_{\text{tot}}(\bm \tau, \bm\mu(\tau)) + \nabla_{\bm a} Q^{\bm\mu}_{\text{tot}}(\bm \tau, \bm a)|_{\bm a = \bm\mu(\bm \tau)} \cdot (\bm a - \bm\mu (\bm\tau)) + \frac{1}{2} \nabla^2 Q^{\bm\mu}_{\text{tot}}(\bm \tau, \bm a_{\zeta}) \parallel \bm a - \pi(\bm \tau) \parallel^2
\end{align*}

Since $\forall \bm a \in O_{\delta}(\pi(\bm \tau))$, we have
\begin{align*}
    |Q^{\bm\mu}_{\text{tot}}(\bm \tau, \bm a) - Q^{\bm\mu}_{\text{tot}}(\bm \tau, \bm\mu(\bm\tau)) - \nabla_{\bm a} Q^{\bm\mu}_{\text{tot}}(\bm \tau, \bm a)|_{\bm a = \bm\mu(\bm \tau)} \cdot (\bm a - \bm\mu (\bm \tau)) | \leq \frac{1}{2}L \delta^2
\end{align*}

Noticing that the first order Taylor expansion of $Q^{\bm\mu}_{\text{tot}}$ has the form $\sum_{[n]} k_i(\bm \tau) Q_i^{\phi}(\bm \tau, a_i) + b(\bm \tau)$. Therefore, the optimal solution of the MSE problem in Eq. \ref{appx:equ:mse} under \name~critics has an error term less than $O(L\delta^2)$ for arbitrary sampling distribution $p(\bm \tau, \bm a)$ of  $\bm a \in O_{\delta}(\bm\mu(\bm \tau))$.

When Q values in the proximity of $\mu(\bm \tau), \forall \bm\tau$ is well estimated within a bounded error and $\delta \ll 1$, approximately, we have

\begin{align} 
    |\frac{\partial Q^{\bm \mu}_{\text{tot}}(\bm \tau, \bm a)}{\partial a_i} - \frac{\partial Q^{\phi}_{\text{tot}}(\bm \tau, \bm a)}{\partial a_i}|& \approx |\frac{Q^{\bm \mu}_{\text{tot}}(\bm \tau, a_{-i}, a_i+\delta) -  Q^{\bm \mu}_{\text{tot}}(\bm \tau, \bm a)}{\delta} - \frac{ Q^{\phi}_{\text{tot}}(\bm \tau, a_{-i}, a_i+\delta) -  Q^{\phi}_{\text{tot}}(\bm \tau, \bm a)}{\delta}| \nonumber \\
    & = |\frac{Q^{\bm \mu}_{\text{tot}}(\bm \tau, a_{-i}, a_i+\delta) -  Q^{\phi}_{\text{tot}}(\bm \tau, a_{-i}, a_i+\delta)}{\delta} -  \frac{ Q^{\bm \mu}_{\text{tot}}(\bm \tau, \bm a) -  Q^{\phi}_{\text{tot}}(\bm \tau, \bm a)}{\delta}| \nonumber \\
    & \sim O(L\delta) \nonumber
\end{align}

\section{Algorithms}\label{appx:algorithm}
In this section, we describe the details of our algorithms, as shown in Algorithm~\ref{alg:s_dop} and~\ref{alg:d_dop}.

\begin{algorithm}
	\caption{Stochastic DOP}\label{alg:s_dop}
	\begin{algorithmic}
	\State Initialize a critic network $Q^{\phi}$, actor networks $\pi_{\theta_i}$, and a mixer network $M^{\psi}$ with random parameters $\phi$,$\theta_i$, $\psi$.
	\State Initialize target networks: $\phi'=\phi$, $\theta'=\theta$, $\psi'=\psi$
	\State Initialize an off-policy replay buffer $\mathcal{D}_\text{off}$ and an on-policy replay buffer $\mathcal{D}_\text{on}$.
	\For{$t=1$ to $T$}
	\State Generate a trajectory and store it in $\mathcal{D}_\text{off}$ and $\mathcal{D}_\text{on}$
    \State Sample a batch consisting of $N_1$ trajectories from $\mathcal{D}_\text{on}$ 
    \State Update decentralized policies using the gradients described in Eq.~\ref{equ:s_dop_g}
    \State Calculate $\mathcal{L}^{\text{On}}(\phi)$
    \State Sample a batch consisting of $N_2$ trajectories from $\mathcal{D}_\text{off}$
    \State Calculate $\mathcal{L}^{\text{\name-TB}}(\phi)$
    \State Update critics using $\mathcal{L}^{\text{On}}(\phi)$ and $\mathcal{L}^{\text{\name-TB}}(\phi)$
    \If{$t\ \text{mod}\  d=0$}
    \State Update target networks: $\phi'=\phi$, $\theta'=\theta$, $\psi'=\psi$
    \EndIf
    \EndFor
	\end{algorithmic}
\end{algorithm}

\begin{algorithm}
	\caption{Deterministic DOP}\label{alg:d_dop}
	\begin{algorithmic}
	\State Initialize a critic network $Q^{\phi}$, actor networks $\mu_{\theta_i}$ and a mixer network $M^{\psi}$ with random parameters $\theta,\phi,\psi$
	\State Initialize target networks: $\phi'=\phi$, $\theta'=\theta$, $\psi'=\psi$
	\State Initialize replay buffer $\mathcal{D}$
	\For{$t=1$ to $T$}
	\State Select action with exploration noise $\va\sim\bm\mu(\bm\tau)+\epsilon$, generate a transition and store the transition tuple in $\mathcal{D}$
    \State Sample $N$ transitions from $\mathcal{D}$ 
    \State Update the critic using the loss function described in Eq.~\ref{equ:d_dpg_c}
    \State Update decentralized policies using the gradients described in Eq.~\ref{equ:d_dop_g}
    \If{$t\ \text{mod}\  d=0$}
    \State Update target networks: $\phi'=\alpha\phi+(1-\alpha)\phi'$, $\theta'=\alpha\theta+(1-\alpha)\theta'$, $\psi'=\alpha\psi+(1-\alpha)\psi'$
    \EndIf
    \EndFor
	\end{algorithmic}
\end{algorithm}
\section{Related works}\label{sec:}

Cooperative multi-agent reinforcement learning provides a scalable approach to learning collaborative strategies for many challenging tasks~\citep{vinyals2019grandmaster, berner2019dota, samvelyan2019starcraft, jaderberg2019human} and a computational framework to study many problems, including the emergence of tool usage~\citep{baker2020emergent}, communication~\citep{foerster2016learning, sukhbaatar2016learning, lazaridou2017multi, das2019tarmac}, social influence~\citep{jaques2019social}, and inequity aversion~\citep{hughes2018inequity}. Recent work on role-based learning~\citep{wang2020roma} introduces the concept of division of labor into multi-agent learning and grounds MARL into more realistic applications. 

Centralized learning of joint actions can handle coordination problems and avoid non-stationarity. However, the major concern of centralized training is scalability, as the joint action space grows exponentially with the number of agents. The coordination graph~\citep{guestrin2002coordinated, guestrin2002multiagent} is a promising approach to achieve scalable centralized learning, which exploits coordination independencies between agents and decomposes a global reward function into a sum of local terms. \citet{zhang2011coordinated} employ the distributed constraint optimization technique to coordinate distributed learning of joint action-value functions. Sparse cooperative Q-learning~\citep{kok2006collaborative} learns to coordinate the actions of a group of cooperative agents only in the states where such coordination is necessary. These methods require the dependencies between agents to be pre-supplied. To avoid this assumption, value function decomposition methods directly learn centralized but factorized global Q-functions. They implicitly represent the coordination dependencies among agents by the decomposable structure~\citep{sunehag2018value, rashid2018qmix, son2019qtran, wang2020learning}. The stability of multi-agent off-policy learning is a long-standing problem.~\citet{foerster2017stabilising, wang2020understanding} study this problem in value-based methods. In this paper, we focus on how to achieve efficient off-policy policy-based learning. Our work is complementary to previous work based on multi-agent policy gradients, such as those regarding multi-agent multi-task learning~\citep{teh2017distral, omidshafiei2017deep} and multi-agent exploration~\citep{wang2020influence}.

Multi-agent policy gradient algorithms enjoy stable convergence properties compared to value-based methods~\citep{gupta2017cooperative, wang2020understanding} and can extend MARL to continuous control problems. COMA~\citep{foerster2018counterfactual} and MADDPG~\citep{lowe2017multi} propose the paradigm of centralized critic with decentralized actors to deal with the non-stationarity issue while maintaining decentralized execution. PR2~\citep{wen2019probabilistic} and MAAC~\citep{iqbal2019actor} extend the CCDA paradigm by introducing the mechanism of recursive reasoning and attention, respectively. Another line of research focuses on fully decentralized actor-critic learning~\citep{macua2017diff, zhang2018fully, yang2018mean, cassano2018multi, suttle2019multi, zhang2019distributed}. Different from the setting of this paper, agents have local reward functions and full observation of the true state in these works. 

\section{Infrastructure, architecture, and hyperparameters}\label{appx:architecture}
Experiments are carried out on NVIDIA P100 GPUs and with fixed hyper-parameter settings, which are described in the following sections.

\subsection{Stochastic \name}
In stochastic \name, each agent has a neural network to approximate its local utility. The local utility network consists of two 256-dimensional fully-connected layers with ReLU activation. Since the critic is not used when execution, we condition local Q networks on the global state $s$. The output of the local utility networks is $Q_i^{\phi_i}(\bm\tau,\cdot)$ for each possible local action, which are then linearly combined to get an estimate of the global Q value. The weights and bias of the linear combination, $k_i$ and $b$, are generated by linear networks conditioned on the global state $s$. $k_i$ is enforced to be non-negative by applying absolute activation at the last layer. We then divide $k_i$ by $\sum_ik_i$ to scale $k_i$ to $[0,1]$.

The local policy network consists of three layers, a fully-connected layer, followed by a 64 bit GRU, and followed by another fully-connected layer that outputs a probability distribution over local actions. We use ReLU activation after the first fully-connected layer.

For all experiments, we set $\kappa=0.5$ and use an off-policy replay buffer storing the latest $5000$ episodes and an on-policy buffer with a size of $32$. We run $4$ parallel environments to collect data. The optimization of both the critic and actors is conducted using RMSprop with a learning rate of $5\times 10^{-4}$, $\alpha$ of 0.99, and with no momentum or weight decay. For exploration, we use $\epsilon$-greedy with $\epsilon$ annealed linearly from $1.0$ to $0.05$ over $500k$ time steps and kept constant for the rest of the training. Mixed batches consisting of $32$ episodes sampled from the off-policy replay buffer and $16$ episodes sampled from the on-policy buffer are used to train the critic. For training actors, we sample $16$ episodes from the on-policy buffer each time. The framework is trained on fully unrolled episodes. The learning rates for the critic and actors are set to $1\times10^{-4}$ and  $5\times10^{-4}$, respectively. And we use $5$-step decomposed multi-agent tree backup. All experiments on StarCraft II use the default reward and observation settings of the SMAC benchmark.

\subsection{Deterministic \name}
The critic network structure of deterministic \name~is similar to that of stochastic \name, except that local actions are part of the input in deterministic \name. For actors, we use a fully-connected forward network with two 64-dimensional hidden layers with ReLU activation, and the output of actors is a local action. We use an off-policy replay buffer storing the latest $10000$ transitions, from which $1250$ transitions are sampled each time to train the critic and actors. The learning rates of both the critic and actors are set to $5\times10^{-3}$. To reduce variance in the updates of actors, we update the actors and target networks only after 2 updates to the critic, as proposed by~\citet{fujimoto2018addressing}. We also use this technique of delaying policy update in all the baselines. For all the algorithms, we run a single environment to collect data, because we empirically find it more sample efficient than parallel environments in the MPE benchmark. RMSprop with a learning rate of $5\times 10^{-4}$, $\alpha$ of 0.99, and with no momentum or weight decay is used to optimize the critic and actors, which is the same as in stochastic \name.

\end{document}